\documentclass[12pt]{article}

\usepackage{times}  
\usepackage{helvet} 
\usepackage{courier}  
\usepackage[hyphens]{url}  
\usepackage{graphicx} 
\usepackage{graphicx}  
\usepackage[round]{natbib}
\renewcommand{\cite}{\citep}

\usepackage{amssymb}
\usepackage{pifont}
\newcommand{\cmark}{\ding{51}}
\newcommand{\xmark}{\ding{55}}

\usepackage{microtype}
\usepackage{graphicx}
\usepackage{subfigure}
\usepackage{amsmath}
\usepackage{booktabs} 
\usepackage{amsthm}
\usepackage{amsmath}
\usepackage{amssymb}
\usepackage[OT1]{fontenc}
\usepackage{bm}
\newcommand{\mathbbm}[1]{\text{\usefont{U}{bbm}{m}{n}#1}}

\DeclareMathOperator*{\argmax}{arg\,max}

\newtheorem{thm}{Theorem}[section]
\newtheorem{lem}{Lemma}[section]
\newtheorem{Definition}{Definition}[section]

\newenvironment{sciabstract}{%
\begin{quote} \baselineskip14pt\small\hfil {\bf Abstract} \hfil\\[3pt]}
{\end{quote}\vspace{6pt}}

\title{On Value Discrepancy of Imitation Learning}

\author
{Tian Xu$^1$, Ziniu Li$^2$, Yang Yu$^{1,\dag}$\\
\normalsize{$^1$National Key Laboratory for Novel Software Technology, Nanjing University, China}\\
\normalsize{$^2$Polixir}\\
\normalsize{emails: xut@lamda.nju.edu.cn, ziniu.li@polixir.ai, yuy@nju.edu.cn.}\\
\normalsize{$^\dag$To whom correspondence should be addressed}
}

\date{}

\begin{document}

\baselineskip16pt

\maketitle 

\begin{sciabstract}
Imitation learning trains a policy from expert demonstrations. Imitation learning approaches have been designed from various principles, such as behavioral cloning via supervised learning, apprenticeship learning via inverse reinforcement learning, and GAIL via generative adversarial learning. In this paper, we propose a framework to analyze the theoretical property of imitation learning approaches based on discrepancy propagation analysis. Under the infinite-horizon setting, the framework leads to the value discrepancy of behavioral cloning in an order of $O\bigl(\frac{1}{(1-\gamma)^2}\bigr)$. We also show that the framework leads to the value discrepancy of GAIL in an order of $O \bigl( \frac{1}{1-\gamma} \bigr)$. It implies that GAIL has less compounding errors than behavioral cloning, which is also verified empirically in this paper. To the best of our knowledge, we are the first one to analyze GAIL's performance theoretically. The above results indicate that the proposed framework is a general tool to analyze imitation learning approaches. We hope our theoretical results can provide insights for future improvements in imitation learning algorithms.
\end{sciabstract}

\section{Introduction}

Sequential decision problems are extremely challenging due to long-term dependency \cite{SuttonBook}. Compared to learning from scratch with reinforcement learning, learning from expert demonstrations (a.k.a, imitation learning) can significantly reduce sample complexity to learn an optimal policy. Successful applications by imitation learning include playing video games \cite{efficient_reduction_IL}, robot control \cite{RobotIL} and  autonomous driving \cite{AutonoumusDrivingIL}.

Imitation learning approaches have been designed from various principles. Behavioral cloning (BC) \cite{BC, bcfromobs, efficient_reduction_IL, dagger} learns a policy via directly minimizing policy (action) discrepancy on each visited state from expert demonstrations. Apprenticeship learning (AL) \cite{abbeel04, MaxEntIRL} infers a reward function from expert demonstrations via inverse reinforcement learning \cite{DBLP:conf/icml/NgR00} and subsequently extracts a policy from the recovered reward function with reinforcement learning. Recently, \citeauthor{GAIL} \cite{GAIL} reveal that AL can be viewed as a dual of state-action occupancy measure matching problem. From this connection, they propose a method called generative adversarial imitation learning (GAIL), which empirically achieves the state-of-art on complicated control tasks. However, little is known about its theoretical property.



In this paper, we focus on the horizon dependency and sample complexity of imitation learning approaches. Since AL is connected with GAIL via dual optimization (see Section \ref{gail_background}), we mainly focus on the analysis of BC and GAIL in this paper. First, we develop a framework to analyze discrepancy propagation in imitation learning. Then we derive the well-known compounding errors \cite{dagger, a_reduction_from_al_to_classification} in BC with the proposed framework. Importantly, we prove that the gap between the value of BC imitator's policy and the expert policy is $O\bigl(\frac{1}{(1-\gamma)^2}\bigr)$ while the gap for GAIL is $O \bigl( \frac{1}{1-\gamma} \bigr)$, where $\gamma$ is the discount factor. We also analyze sample complexity for BC and GAIL. To the best of our knowledge, we are the first one to analyze GAIL's performance theoretically. We hope our theoretical analysis can provide insights for future improvements in imitation learning algorithms.




\begin{table*}[ht]
    \caption{Summary of sample complexity and policy value discrepancy of imitation learning algorithms. The measures of the empirical loss $\epsilon$ are different. BC \cite{BC} and DAgger \cite{dagger} use $\epsilon_{1}$ and $\epsilon_{N}$ denote 0-1 loss where $\epsilon_{1} = \mathbb{E}_{s, a \sim \tau_{E}} [ I (\pi(s) \not = a)]$ and $\epsilon_{T} = \frac{1}{T} \sum_{i=1}^{T} \mathbb{E}_{s, a \sim D_{i} }[ I (\pi(s) \not = a)] $ respectively. FEM \cite{abbeel04} assumes an RL oracle is available. MWAL \cite{a_game_theoretic_approch_to_ap} uses $\epsilon_{R}$ denotes the reward error where $\epsilon_{R} = \max_{s} \bigl| R^{*}(s) - w \cdot \phi(s) \bigr|$. For GAIL, $\epsilon_{\mathcal{D}}$ represents neural distance error $d_{\mathcal{D}}(\hat{\rho}_{\pi_E} , \hat{\rho}_{\pi})$. 
    }
\label{table:comparsion}
\centering
\resizebox{1.0\textwidth}{!}
{
\begin{tabular}{@{}lllll@{}}
    \toprule
    \textbf{Algorithm}       & \textbf{\begin{tabular}[c]{@{}l@{}}Query for\\expert policy\end{tabular}} & \textbf{\begin{tabular}[c]{@{}l@{}}Interact with\\ environment\end{tabular}}                                         & \textbf{\begin{tabular}[c]{@{}l@{}}Sample complexity\end{tabular}} & \textbf{\begin{tabular}[c]{@{}l@{}}Policy value discrepancy\end{tabular}}  \\ \midrule
    
    DAgger         & \cmark                                                                              &  \cmark                                      & $O \left( \frac{\log (1/\delta)}{(1-\gamma)^3 \epsilon^2} \right)$\cite{dagger}                                                                      & $O \left( \frac{1}{1-\gamma} \left( \epsilon_{T} + \frac{1}{1-\gamma} + \sqrt{\frac{\log (1/ \delta) (1-\gamma)}{ m T }}  \right) \right)$\cite{dagger}                                                                                                                                       \\
    \begin{tabular}[c]{@{}l@{}}FEM\end{tabular}  & \xmark     & \cmark & $O \left( \frac{k \log (k/\delta) }{ (1-\gamma)^2 \epsilon^2}  \right) $\cite{abbeel04}                                                                     & $O \left( \frac{1}{1-\gamma} \sqrt{\frac{k \log (k/\delta)}{m}} \right)$\cite{abbeel04}                                                               
        \\
        MWAL  &\xmark
        &\cmark & $O \left( \frac{\log (k/\delta)}{(1-\gamma)^2 \epsilon^2} \right)$ \cite{a_game_theoretic_approch_to_ap}
        & $O \left( \frac{1}{1-\gamma} \sqrt{\frac{\log (k/\delta)}{m}} + \frac{\epsilon_{R}}{1-\gamma} \right)$\cite{a_game_theoretic_approch_to_ap} \\ \midrule
    \begin{tabular}[c]{@{}l@{}}BC \end{tabular}& \xmark                                                 & \xmark                                    &$O \left( \frac{ \vert \mathcal{S} \vert \log \vert \mathcal{A} \vert + \log (1/\delta)}{(1-\gamma)^4 \epsilon ^2 } \right) $  & $O \left( \frac{1}{(1-\gamma)^2} \left( \epsilon_{1} + \sqrt{\frac{\vert \mathcal{S} \vert \log \vert \mathcal{A} \vert + \log(1/\delta)}{m}} \right) \right)$                                              \\
    GAIL       &  \xmark                                                                              &  \cmark                                                       & $O \left( \frac{ \log(1/\delta) }{ (1-\gamma)^2 \epsilon ^2 } \right)$                                                                      & $O \left( \frac{1}{1-\gamma} \left( \sqrt{\epsilon_{\mathcal{D}}} + \sqrt{\hat{\mathcal{R}}_{\rho_{\pi_E}}^{(m)}(\mathcal{D})} + \sqrt{\frac{\log (1/\delta)}{m}}\right)\right)$                                                                                                                                           \\
    \bottomrule
    
    \end{tabular}
}
\end{table*}

This paper is organized as follows. First, we introduce the background and the taxonomy of imitation learning algorithms in Section \ref{sec:background}. Prior works are reviewed in Section \ref{sec:related_work}. In Section \ref{sec:Framework}, we develop a framework to analyze discrepancy propagation for imitation learning approaches. Subsequently, we derive the compounding errors of BC with the proposed method, and analyze behavioral cloning and generative adversarial imitation learning in Section \ref{sec:bc} and Section \ref{sec:GAIL}, respectively. Finally we conduct experiments to validate the theoretical analysis in Section \ref{sec:experiment}. 

\section{Background}
\label{sec:background}

\subsection{Preliminaries}
\textbf{Markov decision progress.} An infinite-horizon\footnote{In this paper, we only consider the infinite-horizon discounted MDP, and it is easy to extend our results into finite-horizon settings.} Markov decision progress is a tuple $M=\left\langle S, A, P, R, \gamma, d_0 \right\rangle$, where $S=\{s_1, \cdots,  s_n\}$ is the state space; $A = \{a_1, \cdots, a_k\}$ is the action space, and $d_0$ specifies the initial state distribution. The sequential decision progress is characterized as follows: at each time $t$, the agent observes a state $s_t$ from the environment and executes an action $a_t$, then the environment sends a reward signal $r(s_t, a_t)$ to the agent and transits to a new state $s_{t+1}$ according to $P(\cdot|s_t, a_t)$.


A stationary policy $\pi(\cdot|s)$ specifies an action distribution conditioned state $s$. The agent is judged by its policy value $V^\pi$ which is defined as the expected discounted cumulative rewards with a discount factor $\gamma \in (0, 1)$.
{
    \small\begin{equation}\label{eq_V_function}
        V^{\pi} = \mathbb{E}_{s_0 \sim d_0,  a_t \sim \pi(\cdot|s_t), s_{t+1} \sim p(\cdot|s_t, a_t)}\left[ \sum_{t=0}^{\infty} \gamma^t r(s_t, a_t)\right]
    \end{equation}
}
The main target of reinforcement learning is to search an optimal policy $\pi^*$ such that it maximizes the policy value (i.e., $\pi^* = \argmax_{\pi} V^\pi$). Complicated tasks like sparse reward settings require a large discount factor $\gamma$ to weight more the future returns. Hence, we represent the horizon dependency in terms of $\gamma$.


To facilitate later analysis, we introduce the discounted state distribution $d_{\pi}(s)$ and discounted state-action distribution $\rho_{\pi}(s, a)$, shown in Eq.\eqref{equation_state_distribution} and Eq.\eqref{equation_state_action_distribution}, respectively. For simplicity, we drop the qualifier "discounted" throughout.

\begin{equation}
\label{equation_state_distribution}
    d_{\pi}(s) = (1-\gamma)\sum_{t=0}^{\infty} \gamma^{t} \operatorname{Pr}\left(s_{t}=s ; \pi, d_{0} \right)
\end{equation}
\begin{equation}
\label{equation_state_action_distribution}
    \rho_{\pi}(s,a) = (1-\gamma)\sum_{t=0}^{\infty}\gamma^{t}\operatorname{Pr}(s_t=s, a_t=a;\pi, d_{0})  
\end{equation}

\textbf{Imitation learning.} Imitation learning (IL) trains a policy from expert demonstrations. In contrast to learning from scratch with reinforcement learning, imitation learning has more information about the optimal policy, thus can significantly reduce sample complexity. Imitation learning approaches have been designed from various principles, such as behavioral cloning \cite{BC,dagger} via supervised learning, apprenticeship learning \cite{abbeel04,SyedBS08} via inverse reinforcement learning \cite{DBLP:conf/icml/NgR00}, and GAIL \cite{GAIL} via generative adversarial learning. In the following, we briefly describe these methods and defer detailed analysis in Section \ref{sec:bc} and Section \ref{sec:GAIL}.



\subsection{Behavioral Cloning}
\label{bc_background}

Behavioral cloning directly mimics expert behaviors by minimizing policy discrepancy. Concretely, BC minimizes the Kullback-Leibler (KL) divergence between the expert policy distribution $\pi_E$ and the learned policy distribution $\pi$ for each state visited by the expert policy.
\begin{equation}
    \min_{\pi} D_{\mathrm{KL}}(\pi_E, \pi) = \min_{\pi} \mathbb{E}_{a \sim \pi_E(\cdot|s)}\left[\log \frac{\pi_E(a|s)}{\pi(a|s)} \right]
\end{equation}

In practice, we often only have access to expert trajectories $\tau_E = \{(s_1, a_1), (s_2, a_2), \cdots\}$ rather than an explicit formula for $\pi_E(\cdot | s)$. Therefore, we optimize the above loss function across state-action pairs contained in expert trajectories, which yields the following optimization problem ($\pi$ is parameterized by $\theta$):
\begin{equation}\label{equation_bc_empircal_loss}
    \min_{\theta} \sum_{(s, a) \sim \tau_E} -\log \pi_\theta(a|s)
\end{equation}


\subsection{Adversary-based Imitation Learning}
\label{gail_background}

In an adversarial learning fashion, apprenticeship learning \cite{abbeel04,SyedBS08} and generative adversarial imitation learning \cite{GAIL} infer a reward function from expert demonstrations and extract a policy with this reward function. But they are distinguished by the means of learning a reward function.

Apprenticeship learning \cite{abbeel04,SyedBS08} infers a reward function that separates expert policy and other policies in terms of policy value. Intuitively, this reward function assigns a high policy value for the expert policy and a low policy value for others. Then the learner maximizes its policy with this reward function to shrink the gap.
\begin{equation}
    \min_\pi \max_{r \in C} \mathbb{E}_{\pi_E}[r(s, a)]  - \mathbb{E}_{\pi}[r(s, a)]
\end{equation}
Where $C$ is class of reward functions. In particular, \citeauthor{abbeel04} \cite{abbeel04} use $C_{\mathrm{linear}} = \{\sum_i w_i f_i: ||w||_2 \leq 1\}$, and \citeauthor{SyedBS08} \cite{SyedBS08} uses $C_{\mathrm{convex}} = \{\sum_i w_i f_i: \sum_i w_i = 1, w_i > 0 \quad \forall i\}$, where $f_i$ is the reward basis function.

Generative adversarial imitation learning \cite{GAIL} also learns a reward function. This reward function is actually a binary-classifier that learns to recognize whether a state-action pair comes from the expert policy. The learner attempts to replicate expert behaviors via maximizing scores given by the classifier.
{
\small\begin{equation}
    \label{equation:gail}
        \min_\pi \max_{D \in (0, 1)^{S \times A}} \mathbb{E}_{ \pi}\left[\log \big(D(s, a) \big)] + \mathbb{E}_{ \pi_E}[\log(1 - D(s, a) \big) \right]
    \end{equation}
}
where $D$ is the binary-classifier (reward function) as mentioned.

Recently, \citeauthor{GAIL} \cite{GAIL} reveal that apprenticeship learning can be viewed as a state-action occupancy matching problem, and the difference between AL and GAIL is the measure of state-action matching.
\begin{equation}
    \min_\pi  \psi^*(\rho_\pi - \rho_{\pi_E})
\end{equation}
where $\psi^*$ is the state-action occupancy measure dependent on the specific solution to the inner problem defined in the original min-max problem. From this dual optimization perspective, the prime problem in GAIL can be recast as $\underset{\pi}{\min} \, D_{\mathrm{JS}} (\rho_{\pi}, \rho_{\pi_E}) $, where $D_{\mathrm{JS}}$ is the Jensen-Shannon (JS) divergence.
\begin{equation}
    D_{\mathrm{JS}}(\rho_\pi, \rho_{\pi_E}) =  \frac{1}{2} \left[ D_{\mathrm{KL}} (\rho_{\pi}, \frac{\rho_{\pi}+\rho_{\pi_E}}{2} )+D_{\mathrm{KL}} (\rho_{\pi_E}, \frac{\rho_{\pi}+\rho_{\pi_E}}{2}) \right]
\end{equation}

\section{Related Work}
\label{sec:related_work}

Learning from scratch with reinforcement learning requires enormous samples to find an optimal policy \cite{Rmax, KearnsS02}. Imitation learning is sample efficient for sequential decision problem via learning from expert demonstrations \cite{BC, dagger, DBLP:conf/icml/NgR00, GAIL}. In this section, we review previous imitation learning algorithms with the focus on their horizon dependency and sample complexity.

Prior works \cite{efficient_reduction_IL, a_reduction_from_al_to_classification, dagger} reveal that behavioral cloning leads to the compounding errors (a quadratic regret concerning horizon length). The reason is that training data generated by the expert policy and testing data generated by the learned policy is not i.i.d as the one in traditional supervised learning. We develop an alternative method to analyze the horizon dependency of BC. Our analysis shares some commons to these results, but highlights that minimizing policy discrepancy in BC naturally works worse under long-horizon settings (see Section \ref{sec:bc}). We also notice that DAgger \cite{dagger} improves the policy value error from $O(T^2)$ to $O(T)$ at the cost of querying additional expert guidance when training, where $T$ is task horizon.

Inverse reinforcement learning is first proposed in \cite{DBLP:conf/icml/NgR00} via recovering a reward function to satisfy Bellman optimality\footnote{Recovering a Bellman optimal policy is strongly strict than recovering a near-optimal policy in terms of policy value.}. Recently \citeauthor{DBLP:journals/correctness/abs-1906-00422}\cite{DBLP:journals/correctness/abs-1906-00422} reformulate this problem as a L1-regularized support vector machine problem and shows the sample complexity of $O\left( (\frac{n\gamma}{(1- \gamma)^2})^2  \log(nk) \right)$. However, the gap of policy value between recovered policy and expert policy is still unknown. Apprenticeship learning (FEM) \cite{abbeel04} and multiplicative-weights apprenticeship learning (MWAL) \cite{SyedBS08} infer a reward function that separates expert policy and other policies in terms of policy value. For FEM \cite{abbeel04}, the sample complexity of $O \left(\frac{k}{\epsilon^2 (1-\gamma)^2} \log k \right)$ is guaranteed to learn a $\epsilon$-optimal policy with the assumption that true reward function lies in linear combinations of reward basis functions, where $k$ is the number of such defined functions. MWAL \cite{SyedBS08} reduces the sample complexity to $O \left(\frac{1}{\epsilon^2 (1-\gamma)^2} \log k \right)$ by the multiplicative weights algorithm. Note that FEM and MWAL are still computationally slow since they solve an RL problem each iteration. Several works \cite{AIRL, GuidedCostLearning} extend apprenticeship learning into high-dimensional problems. Recently, \citeauthor{GAIL}\cite{GAIL} deduce the dual of maximum causal entropy IRL \cite{MaxEntIRL, ZiebartBD10}, upon which they describe an algorithm called generative adversarial imitation learning (GAIL). \citeauthor{GAIL} \cite{GAIL} show that ideally GAIL aims to search a policy such that it minimizes the Jensen-Shannon divergence to the expert in terms of state-action occupancy measure. Importantly, we find that this objective function theoretically results in less horizon dependency.

Theoretically analyzing algorithms with non-linear function approximation like GAIL\cite{GAIL} extremely difficult. Our analysis relies on the recently proposed generalization theory \cite{Generalization_in_GAN, D-G_tradeoff} for generative adversarial networks \cite{GAN}. However, we are interested in the policy value induced by the generator (policy) under the Markov decision process settings. Besides, we note that performance differences between BC \cite{BC} and GAIL \cite{GAIL} can be attributed to the used discrepancy measure. As discussed in \cite{KL_JS}, generative models based on $D_{\mathrm{KL}}(P,Q)$ tends to fit models $Q$ that cover all modes of $P$ while models based on $D_{\mathrm{JS}}(P,Q)$ can generate nature-look images with the punishment that forces $Q$ concentrate around the largest mode of $P$. As we have emphasized, this result, however, cannot directly be applied to imitation learning settings due to the nature of horizon dependency.

We summarize the theoretical properties of the mentioned imitation learning algorithms in Table \ref{table:comparsion}. The policy value discrepancy of BC \cite{BC} is $O\big(\frac{1}{(1-\gamma)^2}\big)$, and others achieve $O\big(\frac{1}{1-\gamma}\big)$. Note the cost of the optimization problem is different. DAgger \cite{dagger} achieves this by querying more expert guidance when training, while adversary-based algorithms like FEM \cite{abbeel04}, MWAL \cite{SyedBS08}, and GAIL \cite{GAIL} require interactions to solve the min-max problem.

\section{Discrepancy Analysis in IL}\label{sec:Framework}

In this section, we develop a framework to derive policy value discrepancy $|V^{\pi_{E}} - V^{\pi}|$ for various imitation learning algorithms. We trace the source of policy value discrepancy and characterize the relationship between different discrepancy measures shown in Figure \ref{fig:value}.


\begin{figure}[ht]
    \centering
    \includegraphics[width=0.45\textwidth]{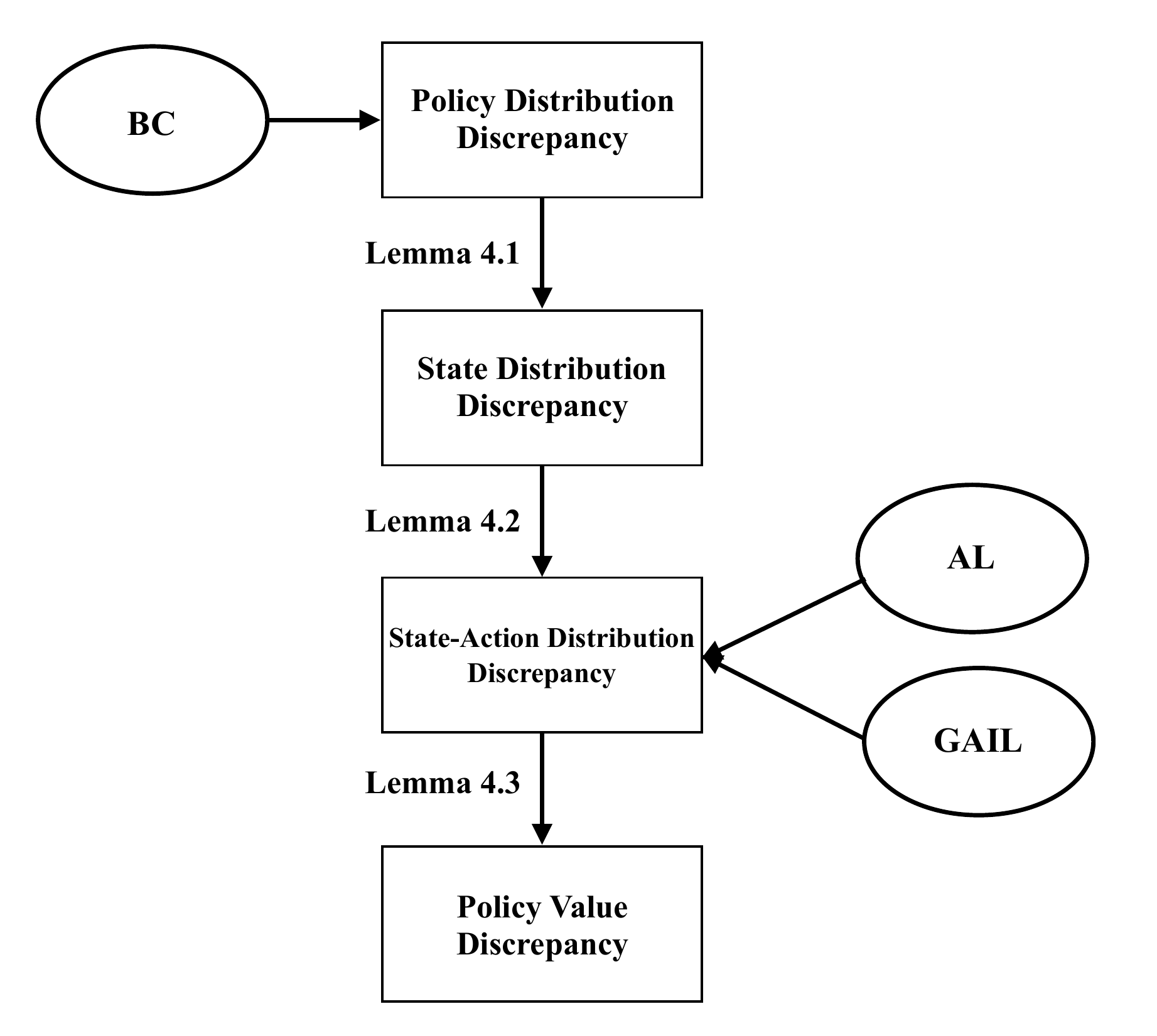}
    \caption{Discrepancy propagation in imitation learning}
    \label{fig:value}
\end{figure}

\begin{lem}\label{leamma_s_cond}
    Let total variation between two distributions defined as $D_{\mathrm{TV}}(P, Q) = \frac{1}{2} ||P - Q||_1$. Let $\pi_{E}$ denote the expert policy and $\pi$ denote the imitator's policy. Then the total variation between $d_{\pi_E}$ and $d_{\pi}$ is bounded by the expectation of total variation between $\pi(\cdot|s)$ and $\pi_{E}(\cdot|s)$ over state distribution $d_{\pi_E}$.
    \\
    \begin{center}
    $
    D_{\mathrm{TV}}(d_{\pi}, d_{\pi_{E}}) \leq \frac{\gamma}{1-\gamma}\mathbb{E}_{s \sim d_{\pi_E}}\biggl[ D_{\mathrm{TV}} \big( \pi (\cdot|s), \pi_{E}(\cdot|s) \big) \biggr]$
    \end{center}
\end{lem}

This Lemma suggests the relationship between policy discrepancy and state distribution discrepancy, and the proof is left in Appendix \ref{appendix_for_proof}. Intuitively, a disagreement of decision at state $s_t$ may result in a different $s_{t+1}$. As this process repeats, state distribution discrepancy accumulates over time steps. Based on Lemma \ref{leamma_s_cond}, we further derive the relationship between state-action distribution discrepancy and policy discrepancy.

\begin{lem}\label{lemma_bc_sa_con}
The total variation between two state-action distributions $D_{\mathrm{TV}}(\rho_{\pi},\rho_{\pi_{E}})$ is bounded by the expectation of total variation between $\pi(\cdot|s)$ and $\pi_{E}(\cdot|s)$ over state distribution $d_{\pi_E}$:
    
\begin{equation}
D_{\mathrm{TV}}(\rho_{\pi}, \rho_{\pi_{E}}) \leq \frac{1}{1-\gamma}\mathbb{E}_{s \sim d_{\pi_E}} \bigl[ D_{\mathrm{TV}} \bigl(\pi(\cdot|s), \pi_{E}(\cdot|s) \bigr) \bigr]
\end{equation}
\end{lem}

\begin{proof}
Recall the definition of $\rho_{\pi}$ in Eq.\eqref{equation_state_action_distribution}, we have
\begin{align}\label{joint_1}
        & \qquad D_{\mathrm{TV}}(\rho_{\pi}, \rho_{\pi_{E}}) \\ &= \frac{1}{2} \sum_{s, a} \bigl| \pi_{E}(a|s) d_{\pi_E}(s) - \pi(a|s) d_{\pi}(s) \bigr|
        \nonumber\\
        &= \frac{1}{2} \sum_{s, a} \bigl| \bigl[\pi_{E}(a|s) - \pi(a|s)\bigr] d_{\pi_{E}}(s)  +\bigl[d_{\pi_E}(s) - d_{\pi}(s) \bigr]\pi(a|s) \bigr|
        \nonumber\\
        &\leq \frac{1}{2} \sum_{s, a} \bigl| \pi_{E}(a|s) - \pi(a|s) \bigr|d_{\pi_{E}}(s)
        + \frac{1}{2} \sum_{s, a}\pi(a|s) \bigl| d_{\pi_{E}}(s) - d_{\pi}(s) \bigr|  \qedhere 
\end{align}
It is easy to verify that the first term is the total variation between two policy distributions $\pi(\cdot|s)$ and $\pi_{E}(\cdot|s)$: 
\begin{equation}\label{cond_1}
    \begin{split}
        \frac{1}{2} \sum_{s, a} \bigl| \pi_{E}(a|s) - \pi(a|s) \bigr|d_{\pi_{E}}(s)
        &= \sum_{s} d_{\pi_E}(s) \sum_{a} \frac{1}{2} \bigl| \pi_{E}(a|s) - \pi(a|s) \bigr|  
        \\
        &= \mathbb{E}_{s \sim d_{\pi_E} } [D_{\mathrm{TV}}(\pi(\cdot|s) , \pi_{E}(\cdot|s) )]
    \end{split}
\end{equation}
The second term is the total variation between state distribution  $d_{\pi}$ and $d_{\pi_E}$:
\begin{equation}\label{state_1}
    \begin{split}
        \frac{1}{2} \sum_{s, a}\pi(a|s) \bigl| d_{\pi_{E}}(s) - d_{\pi}(s) \bigr| &= \frac{1}{2} \sum_{s}  \bigl| d_{\pi_{E}}(s) - d_{\pi}(s) \bigr| \sum_{a} 
        \pi(a|s)
                \\
        &= D_{\mathrm{TV}}( d_{\pi}, d_{\pi_{E}})
    \end{split}
\end{equation}
Combining Eq.\eqref{cond_1}, \eqref{state_1} with \eqref{joint_1}, we get the following result based on Lemma \ref{leamma_s_cond}.
\begin{equation*}
    \begin{split}
        D_{\mathrm{TV}}(\rho_{\pi}, \rho_{\pi_{E}}) &\leq 
        \mathbb{E}_{s \sim d_{\pi_E} } [D_{\mathrm{TV}}(\pi_{E}(\cdot|s), \pi(\cdot|s))]
        + D_{\mathrm{TV}}(d_{\pi}, d_{\pi_{E}})
        \\
        &\leq  (1 + \frac{\gamma}{1-\gamma})\mathbb{E}_{s \sim d_{\pi_E} } [D_{\mathrm{TV}}( \pi(\cdot|s), \pi_{E}(\cdot|s))]
        \\
        &= \frac{1}{1-\gamma}\mathbb{E}_{s \sim d_{\pi_E} } [D_{\mathrm{TV}}( \pi(\cdot|s), \pi_{E}(\cdot|s))]\qedhere
    \end{split} 
\end{equation*}
\end{proof}
Similarly, Lemma \ref{lemma_bc_sa_con} indicates that optimizing one-step policy discrepancy naturally introduces a horizon-dependent term $\frac{1}{1-\gamma}$. To build the policy value gap with state-action distribution discrepancy, we reformulate the policy value defined in Eq.\eqref{eq_V_function} with an alternative representation. 
\begin{equation}\label{equation_policy_value}
    V^\pi = \frac{1}{1-\gamma} \mathbb{E}_{\rho_\pi}[r(s, a)]
\end{equation}
where the denominator is to compensate the normalization constant induced in Eq.\eqref{equation_state_action_distribution}.
\begin{lem}\label{lemma_bc_value}
    Assume that reward function is bounded in absolute value $R_{\max}$. Then the policy value discrepancy is bounded by the state-action distribution discrepancy.
    \begin{displaymath}
        \begin{aligned}
            \bigl | V^{\pi} - V^{\pi_{E}} \bigr| \leq \frac{2 R_{\max} }{1-\gamma}  D_{\mathrm{TV}}(\rho_{\pi}, \rho_{\pi_{E}})  
        \end{aligned}  
    \end{displaymath}
\end{lem}
\begin{proof}
    By the Eq.\eqref{equation_policy_value}, we have that
    \begin{equation}\label{equation_bc_v1}
        \begin{split}
            \bigl| V^{\pi} - V^{\pi_{E}} \bigr| &= \frac{1}{1-\gamma} \bigl| \sum_{s, a} \bigl[ \rho_{\pi}(s,a) - \rho_{\pi_{E}}(s,a) \bigr] r(s,a) \bigr|
            \\
            &\leq \frac{1}{1-\gamma}\sum_{s, a}\bigl| \rho_{\pi}(s,a) - \rho_{\pi_{E}}(s,a) \bigr| r(s,a)
            \\
            &\leq \frac{2R_{\max}}{1-\gamma}D_{\mathrm{TV}}(\rho_{\pi}, \rho_{\pi_{E}})
        \end{split}\qedhere
    \end{equation}
\end{proof}

It turns out that state-action discrepancy plays an important role in analyzing the policy value discrepancy later. In the following, we will utilize this framework to analyze imitation learning approaches. Since apprenticeship learning is connected with GAIL via dual optimization as discussed previously, we mainly focus on the analysis of BC and GAIL in this paper.

\section{Behavioral Cloning}
\label{sec:bc}

In this section, we first deduce the compounding errors \cite{dagger, a_reduction_from_al_to_classification}. Subsequently, we show the sample complexity of BC.

\subsection{Horizon Dependency}

With Lemma \ref{lemma_bc_sa_con} and Lemma \ref{lemma_bc_value}, we can easily build up the relationship between the policy discrepancy and policy value discrepancy.
    

\begin{thm}\label{theorem_value_bc}
    Let $\pi_{E}$ and $\pi_{bc}$ denote the expert policy and BC imitator's policy. Assume that reward function is bounded in absolute value $R_{\max}$. Then the BC imitator has policy value error
    \begin{equation}\label{equation_bc_value_discrepancy}
        \begin{split}
            \bigl| V^{\pi_{bc}} - V^{\pi_{E}} \bigr| \leq \frac{2 R_{\max}}{(1-\gamma)^2}\mathbb{E}_{s \sim d_{\pi_E}}[D_{\mathrm{TV}}(\pi_{bc}(\cdot|s), \pi_{E}(\cdot|s))]
        \end{split}
    \end{equation}
    \end{thm}
    
    Theorem \ref{theorem_value_bc} implies a quadratic policy value gap for behavioral cloning in terms of the horizon. One can understand this result by imaging the case that learned policy may visit states which are not covered in expert behaviors. In that case, the policy value gap accumulates quadratically in the horizon length. We underline that without considering temporal structure, objective function based on one-step policy discrepancy should be used carefully in MDP with long-horizon settings.
    
    
    Though we derive the above results from a different perspective, our results are consistent with previous works \cite{efficient_reduction_IL, dagger}. In particular, the quadratic bound in Theorem \ref{theorem_value_bc} can be viewed as an extension of \cite{dagger, a_reduction_from_al_to_classification} to infinite-horizon settings. Note that our results are not restricted to the metric of total variation. Here, we briefly discuss the settings based on KL- divergence. It is known that $D_{\mathrm{TV}}(P, Q) \leq \sqrt{\frac{1}{2} D_{\mathrm{KL}}(P, Q)}$ and applying this inequality into Eq.\eqref{equation_bc_value_discrepancy} yields
    \begin{align*}
        \left| V^{\pi_{bc}} - V^{\pi_{\pi_{E}}} \right| \leq \frac{ \sqrt{2} R_{\max} }{(1-\gamma)^2} \mathbb{E}_{s \sim d_{\pi_E}}\left[\sqrt{ D_{\mathrm{KL}} (\pi_{bc}, \pi_{E})}\right]
    \end{align*}
    If we let $\epsilon = D_{\mathrm{KL}} (\pi_{bc}, \pi_{E})$, the policy value error can be bounded by $O \left( \sqrt{\epsilon} \right)$, which is also consist with \cite{a_reduction_from_al_to_classification}.

\subsection{Sample Complexity}

In this section, we analyze the sample complexity of BC. We start with the generalization error with respect to state-action distribution, then gives the sample complexity in terms of policy value discrepancy.

\begin{lem}\label{lemma_bc_pac}
Let $\Pi$ be the set of all deterministic policy and $\vert \Pi \vert = {\vert A \vert} ^{\vert S \vert} = k^n$. Assume that there does not exist a policy $\pi \in \Pi$ such that $\pi(s^{(i)}) = a^{(i)} \; \forall i\in\{1, \dots, m\}$. Then, for any $\delta > 0 $, with probability at least $1-\delta$, the following inequality holds:
\begin{equation*}\label{bc_bound_1}
    \begin{split}
        \mathbb{E}_{s \sim {d_{\pi_E}}} \bigl[ D_{\mathrm{TV}} \big( \pi_{bc}(\cdot|s), \pi_{E}(\cdot|s) \big) \bigr] \leq & \frac{1}{m}\sum_{i=1}^{m} I \bigl[\pi_{bc} (s^{(i)}) \not= a^{(i)} \bigr] + \\
        & \sqrt{\frac{\log\vert \Pi \vert + \log (2/\delta) }{2m}}
    \end{split}
\end{equation*}
\end{lem}

The proof is left in Appendix \ref{appendix_bc}. The left side in Lemma \ref{lemma_bc_pac} is the generalization error which is bounded by two terms. The first term is the empirical error on the training dataset, which dependent on detailed supervised learning algorithms. The second term is about model complexity and the number of training samples, which implies that a greater state space $S$ and action space $A$ incur more generalization error.


\begin{thm}\label{theorem_bc_value}
For any $\delta > 0$, with probability at least $1-\delta$, the following inequality holds:
\begin{equation*}
    \begin{split}
        \biggl| V^{\pi_{bc}} - V^{\pi_{E}} \biggr| \leq \frac{2 R_{\max}}{(1-\gamma)^2} \biggl( & \frac{1}{m}\sum_{i=1}^{m} I \biggl[\pi_{bc}(s^{(i)}) \not= a^{(i)}\biggr] + \\
    &\sqrt{\frac{\log \vert \Pi \vert + \log(2/\delta)} {2m}} \biggr)
    \end{split}
\end{equation*}
\end{thm}

Theorem \ref{theorem_bc_value} can be easily derived from Lemma \ref{lemma_bc_value} and Lemma \ref{lemma_bc_pac}, thus the proof is omitted. Apparently, Theorem \ref{theorem_bc_value} shows the policy value discrepancy is dependent on the number of expert demonstrations $m$ and the size of policy class $\vert \Pi \vert$. Though this result resembles the generalization error of traditional supervised learning, we highlight the quadratic horizon dependency term for imitation learning where decisions are temporally related.

\section{Generative Adversarial Imitation Learning}
\label{sec:GAIL}
Unlike apprenticeship learning algorithms \cite{abbeel04, a_game_theoretic_approch_to_ap}, we have little knowledge about the theoretical property of GAIL. For simplicity, we assume that the discriminator is optimal in this paper. With this assumption, the policy $\pi_{GA}$ in GAIL is to optimize the Jensen-Shannon (JS) divergence.




\subsection{Horizon Dependency}

\begin{thm}\label{theorem_gail_value_long_term}
Let $\pi_{E}$ and $\pi_{GA}$ denote the expert policy and GAIL imitator's policy. The reward function is bounded by $R_{\max}$. Then GAIL imitator has the following policy value error.
\begin{align*}
    \left| V^{\pi_E} - V^{\pi_{GA}} \right| \leq \frac{2 \sqrt{2} R_{\max}}{1-\gamma} \sqrt{D_{\mathrm{JS}} \left( \rho_{\pi_{GA}}, \rho_{\pi_E} \right)}
\end{align*}
\end{thm}
\begin{proof}
Firstly, We show the connection between $D_{\mathrm{TV}} (\rho_{\pi_{GA}}, \rho_{\pi_E})$ and $D_{\mathrm{JS}} (\rho_{\pi_{GA}}, \rho_{\pi_E})$.
\begin{equation}\label{equation_connection_JS_TV}
    \begin{split}
        &D_{\mathrm{JS}} \left( \rho_{\pi_{GA}}, \rho_{\pi_E} \right)
        \\
        &= \frac{1}{2} \left( D_{\mathrm{KL}} \left(\rho_{\pi_{GA}}, \frac{\rho_{\pi_{GA}} + \rho_{\pi_E}}{2} \right) + D_{\mathrm{KL}} \left(\rho_{\pi_{E}}, \frac{\rho_{\pi_{GA}} + \rho_{\pi_E}}{2} \right) \right)
        \\
        &\geq D_{\mathrm{TV}}^{2} \left( \rho_{\pi_{GA}}, \frac{\rho_{\pi_{GA}} + \rho_{\pi_E}}{2} \right) +
        D_{\mathrm{TV}}^{2} \left( \rho_{\pi_{E}}, \frac{\rho_{\pi_{GA}} + \rho_{\pi_E}}{2} \right)
        \\
        &= \frac{1}{2} D_{\mathrm{TV}}^{2} \left( \rho_{\pi_{GA}}, \rho_{\pi_E} \right)
    \end{split}
\end{equation}
Based on the connection between stat-action distribution discrepancy and policy value discrepancy, we show the policy value error bound for GAIL.
\begin{equation*}
    \begin{split}
        \left| V^{\pi_E} - V^{\pi_{GA}} \right| &\leq \frac{2 R_{\max}}{1-\gamma} D_{\mathrm{TV}} \left( \rho_{\pi_E}, \rho_{\pi_{GA}} \right)
        \\
        &\leq \frac{2 \sqrt{2} R_{\max}}{1-\gamma} \sqrt{D_{\mathrm{JS}} \left( \rho_{\pi_{GA}}, \rho_{\pi_E} \right)}
    \end{split}
\end{equation*}
\end{proof}
Theorem \ref{theorem_gail_value_long_term} indicates that the value error bound for GAIL grows linearly with the horizon term $\frac{1}{1-\gamma}$. However, the cost is that GAIL must interact with the environment to optimize $D_{\mathrm{JS}}(\rho_{\pi_{GA}}, \rho_{\pi_{E}})$ with reinforcement learning. In addition, we also notice that GAIL also enjoys $O \left( \sqrt{\epsilon} \right)$ like behavioral cloning, if we are allowed to define $\epsilon=D_{\mathrm{JS}} \left( \rho_{\pi_{GA}}, \rho_{\pi_E} \right)$ for GAIL.

\subsection{Sample Complexity}

Analyzing generalization ability and sample complexity of GAIL is somewhat more complicated. Unlike behavioral cloning, GAIL simultaneously trains two models: a policy model $\pi_{GA}$ that imitates the expert policy, a discriminative model $D$ that distinguishes the state-action pairs from $\pi_{GA}$ and $\pi_E$. For behavioral cloning, we can directly optimize the policy parameters (See Eq.\eqref{equation_bc_empircal_loss}). However, we can only optimize the GAIL via samples from the policy distribution rather than the policy distribution parameters. 

Based on the generalization theory in GAN \cite{Generalization_in_GAN, D-G_tradeoff}, we define the generalization in GAIL as follows:
\begin{Definition}\label{def_generalization}
Given $\hat{\rho}_{\pi_E}$, the empirical distribution over \\ $\{ (s^{(i)}, a^{(i)}) \}_{i=1}^{m}$ obtained by $\pi_E$, a state-action distribution $\rho_{\pi}$ generalizes under the distance between distributions $d(\cdot,\cdot)$ with error $\epsilon$ if with high probability, the following inequality holds.
\begin{align*}
    \left| d(\rho_{\pi}, \rho_{\pi_E}) - d(\hat{\rho}_{\pi}, \hat{\rho}_{\pi_E})  \right| \leq \epsilon
\end{align*}
Where $\hat{\rho}_{\pi}$ is the empirical distribution of $\rho_{\pi}$ with m samples \\ $\{ (s^{(i)}, a^{(i)}) \}_{i=1}^{m}$ obtained by $\pi$.
\end{Definition}

We are interested in bounding the distance between $\rho_{\pi}$ and $\rho_{\pi_E}$ with certain distance metric. \citeauthor{Generalization_in_GAN}  prove that JS divergence doesn't generalize with any number of examples because the true distance $D_{\mathrm{JS}}({\rho}_{\pi}, \rho_{\pi_E})$ is not reflected by the empirical distance $D_{\mathrm{JS}}(\hat{\rho}_{\pi}, \hat{\rho}_{\pi_E})$. This phenomenon also happens in generative adversary learning algorithms (see \cite{Generalization_in_GAN} for more details). Hence, for our analysis, we choose the neural net distance $d_{\mathcal{D}} (\mu, \nu)$, which turns out that neural network distance has a much better generalization properties than Jensen-Shannon divergence. More importantly, it is tractable to build the bound of $D_{\mathrm{TV}} (\rho_{\pi_{GA}}, \rho_{\pi_{E}})$ and the corresponding policy value error $\left| V^{\pi_{GA}} - V^{\pi_{E}} \right|$ via the neural net distance. Firstly, we give the definition of neural net distance as follows.
\begin{Definition}
Let $\mathcal{D}$ denote a class of neural nets. Then the \textit{neural net distance} $d_{\mathcal{D}}(\mu, \nu)$ between two distributions $\mu$ and $\nu$ is defined as
\begin{align*}
        d_{\mathcal{D}}(\mu, \nu) = \underset{D \in \mathcal{D}}{\sup} \mathbbm{E}_{x \sim \mu}[D(x)] - \mathbbm{E}_{x \sim \nu}[D(x)]
\end{align*}
\end{Definition}
With the neural net distance, GAIL-imitator finds a policy by optimizing the following objective.
\begin{align*}
        \underset{\pi \in \Pi}{ \min } \biggl\{ d_{\mathcal{D}}(\hat{\rho}_{\pi_E}, \rho_{\pi}) := \underset{D \in \mathcal{D}}{\sup} \{ \mathbbm{E}_{s, a \sim \hat{\rho}_{\pi_E} }[D(s,a)] - \mathbbm{E}_{s, a \sim \rho_{\pi} }[D(s,a)]   \} \biggr\}
\end{align*}
Where $\mathcal{D}$ is the set of discriminator neural nets, and $\Pi$ is the set of policy nets. Given the definition of generalization and neural net distance, we show that the neural net distance between stat-action joint distributions is bounded.
\begin{lem}\label{lemma_gail_neural_distance}
Assume that the policy $\pi_{GA}$ optimizes GAIL objective $d_{\mathcal{D}}(\hat{\rho}_{\pi_E}, \hat{\rho}_{\pi})$ up to an $\epsilon$ error and the discriminator set $\mathcal{D}$ consists of bounded functions with $\Delta$, i.e. $\| D \|_{\infty} \leq \Delta \;, \forall D \in \mathcal{D}$. Then with probability at least $1-\delta$, the following inequality holds.
\begin{align*}
    d_{\mathcal{D}} (\rho_{\pi_{GA}}, \rho_{\pi_E}) \leq \underset{\pi \in \Pi}{\inf} d_{\mathcal{D}}(\rho_{\pi}, \rho_{\pi_{E}}) + \epsilon + 4\hat{\mathcal{R}}_{\rho_{\pi_E}}^{(m)}(\mathcal{D}) + 2 \Delta \sqrt{\frac{2 \log(1/\delta)}{m}} 
\end{align*}
Where $\hat{\mathcal{R}}_{\rho_{\pi_E}}^{(m)}(\mathcal{D})$ is empirical Rademacher complexity of $\mathcal{D}$ and\\ $\hat{\mathcal{R}}_{\rho_{\pi_E}}^{(m)}(\mathcal{D}) = \mathbbm{E}_{{\bm{\sigma}}} \left[ \sup_{D \in \mathcal{D}} \sum_{i=1}^{m} \frac{1}{m} \sigma_{i} D(X_{i}) \right]$. 
\end{lem}
See Appendix \ref{appendix_gail} for the proof. Lemma \ref{lemma_gail_neural_distance} connects the upper bound of $d_{\mathcal{D}}(\rho_{\pi_{GA}}, \rho_{\pi_E})$ with the empirical Rademacher complexity of the discriminator class $\mathcal{D}$. Under a limited set of expert demonstrations and the same training error $\epsilon$, the more complex the discriminator set $\mathcal{D}$ is, the greater the distance $d_{\mathcal{D}}(\rho_{\pi_{GA}}, \rho_{\pi_E})$ is. The intuition is that when the discriminator class $\mathcal{D}$ is too complex, optimizing the empirical distance $d_{\mathcal{D}}(\hat{\rho}_{\pi}, \hat{\rho}_{\pi_E})$ can not optimize the population distance $d_{\mathcal{D}}(\rho_{\pi}, \rho_{\pi_E})$\cite{Generalization_in_GAN}.

Having bounded the neural distance $d_{\mathcal{D}}(\rho_{\pi}, \rho_{\pi_E})$, the following Lemma bridges neural distance and total variation, which helps us extend the results into total variation.
\begin{lem}
Assume that $\mu$ and $\nu$ have positive density function and the neural net class $\mathcal{D}$ consists of bounded function with $\Delta$. Then
\begin{align*}
    \frac{1}{\Delta} d_{\mathcal{D}}(\mu, \nu) \leq D_{\mathrm{TV}} (\mu, \nu) \leq \sqrt{2 \Lambda_{\mathcal{F}, \Pi} d_{\mathcal{D}}(\mu, \nu)}
\end{align*}
Where $\Lambda_{\mathcal{D}, \Pi} = \underset{\pi \in \Pi}{\sup} \| \log (\frac{\rho_{\pi}}{\rho_{\pi_E}}) \|_{\mathcal{D}, 1} < \infty$ and\\
    {\small$\| \log (\frac{\rho_{\pi}}{\rho_{\pi_E}}) \|_{\mathcal{D}, 1} = \inf \bigl\{ \sum_{i=1}^{n} \vert w_{i} \vert : \log(\frac{\rho_{\pi}}{\rho_{\pi_E}}) = \sum_{i=1}^{n} w_{i} D_{i} + w_{0}, \forall n \in \mathbb{N}, w_{0}, w_{i} \in \mathbb{R}, D_{i} \in \mathcal{D}  \bigr\}$}.
\end{lem}
Based on the connection between neural distance $d_{\mathcal{D}}(\rho_{\pi}, \rho_{\pi_E})$ and total variation $D_{\mathrm{TV}}(\rho_{\pi}, \rho_{\pi_E})$ , we give the bound of total variation $D_{\mathrm{TV}}(\rho_{\pi}, \rho_{\pi_E})$ in the following lemma. 
\begin{lem}\label{LemmaGailTV}
Assume that the policy $\pi_{GA}$ optimizes GAIL objective $d_{\mathcal{D}}(\hat{\rho}_{\pi_E}, \hat{\rho}_{\pi)}$ up to an $\epsilon$ error and all discriminator nets in $\mathcal{D}$ are bounded by $\Delta$. $\hat{\mathcal{R}}_{\rho_{\pi_E}}^{(m)}(\mathcal{D})$ is the empirical Rademacher complexity of $\mathcal{D}$. Then with probability at least $1-\delta$, the following inequality holds.
\begin{equation*}
D_{\mathrm{TV}}(\rho_{\pi_{GA}}, \rho_{\pi_E}) \leq \sqrt{2 \Lambda_{\mathcal{F}, \Pi}}
        \biggl( \underset{\pi \in \Pi}{\inf} \sqrt{\Delta D_{\mathrm{TV}}(\rho_{\pi}, \rho_{\pi_E})} + \sqrt{\epsilon} 
        + 2 \sqrt{\hat{\mathcal{R}}_{\rho_{\pi_E}}^{(m)}(\mathcal{D})} + 2 \Delta \sqrt{\frac{2 log(1/\delta)}{m}}  \biggr)
\end{equation*}

\end{lem}
From the connection between state-action distribution discrepancy and policy value discrepancy in Section \ref{sec:Framework}, we show the policy value error bound dependent on discount factor, number of samples and model complexity.
\begin{thm}\label{theorem_gail}
Assume that the policy $\pi_{GA}$ optimizes the GAIL objective $d_{\mathcal{D}}(\hat{\rho}_{\pi_E}, \hat{\rho}_{\pi})$ up to an $\epsilon$ error and all discriminator nets in $\mathcal{D}$ are bounded by $\Delta$. $\hat{\mathcal{R}}_{\rho_{\pi_E}}^{(m)}(\mathcal{D})$ is the empirical Rademacher complexity of $\mathcal{D}$. Assume that the reward function is bounded by $R_{\max}$. Then with probability at least $1-\delta$, the following inequality holds.
\begin{align*}
   & \left| V^{\pi_{GA}} - V^{\pi_E} \right| \leq\\
   & \frac{2 R_{\max}  \sqrt{2 \Lambda_{\mathcal{D}, \Pi}}}{1-\gamma}
        \biggl( \underset{\pi \in \Pi}{\inf} \sqrt{\Delta \sqrt{2 D_{\mathrm{JS}}(\rho_{\pi}, \rho_{\pi_{E}})}} + \sqrt{\epsilon}
        + 2 \sqrt{\hat{\mathcal{R}}_{\rho_{\pi_E}}^{(m)}(\mathcal{D})} + 2 \Delta \sqrt{\frac{2 \log(1/\delta)}{m}} \biggr)
\end{align*}
\end{thm}
The upper bound in Theorem \ref{theorem_gail} has four terms. The first two terms represent the empirical loss for policy $\pi_{GA}$ which decreases as the training process repeats. The last two terms suggest the generalization ability of GAIL. Theorem \ref{theorem_gail} suggests that controlling the model complexity can improve the performance via avoiding overfitting on the empirical distribution, observed by many practical algorithms \cite{var_dis_bottleneck}. Theorem \ref{theorem_gail} implies that the discriminator class $\mathcal{D}$ should be complex enough to distinguish between $\rho_{\pi}$ and $\rho_{\pi_E}$, striking a trade-off with the requirement that $\mathcal{D}$ should be simple enough to be generalizable. We hope this results may provide insights for future improvements in imitation learning algorithms.

\section{Experiments}\label{sec:experiment}

In this section, we conduct experiments to validate the previous theoretical results. Here, we focus on the horizon dependency and sample complexity of imitation learning algorithms. We evaluate imitation learning methods on three Mujoco tasks: \textit{Ant}, \textit{Hopper}, and \textit{Walker}. Reported results are based on the true reward function defined in the OpenAI Gym \cite{Gym}. We consider the following approaches: GAIL \cite{GAIL}, BC \cite{BC}, DAgger \cite{dagger} and apprenticeship learning algorithms described below. As we stated previously, it is computationally expensive to run FEM \cite{abbeel04} and MWAL \cite{SyedBS08}. Following \cite{GAIL}, we consider the accelerated algorithms proposed in \cite{Jonathan16}. In particular, we test FEM, the algorithm of \cite{Jonathan16} using the linear reward function class in \cite{abbeel04}, and GTAL, the algorithm of \cite{Jonathan16} using the convex reward function class of \cite{SyedBS08}. In reality, we cannot simulate infinite-horizon settings, thus we truncate the episode length into $1000$. We report the empirical policy value by Monte Carlo simulation with $20$ trajectories. All experiments run five seeds.  

\begin{figure}[t]
    \centering
    \includegraphics[width=0.5\linewidth]{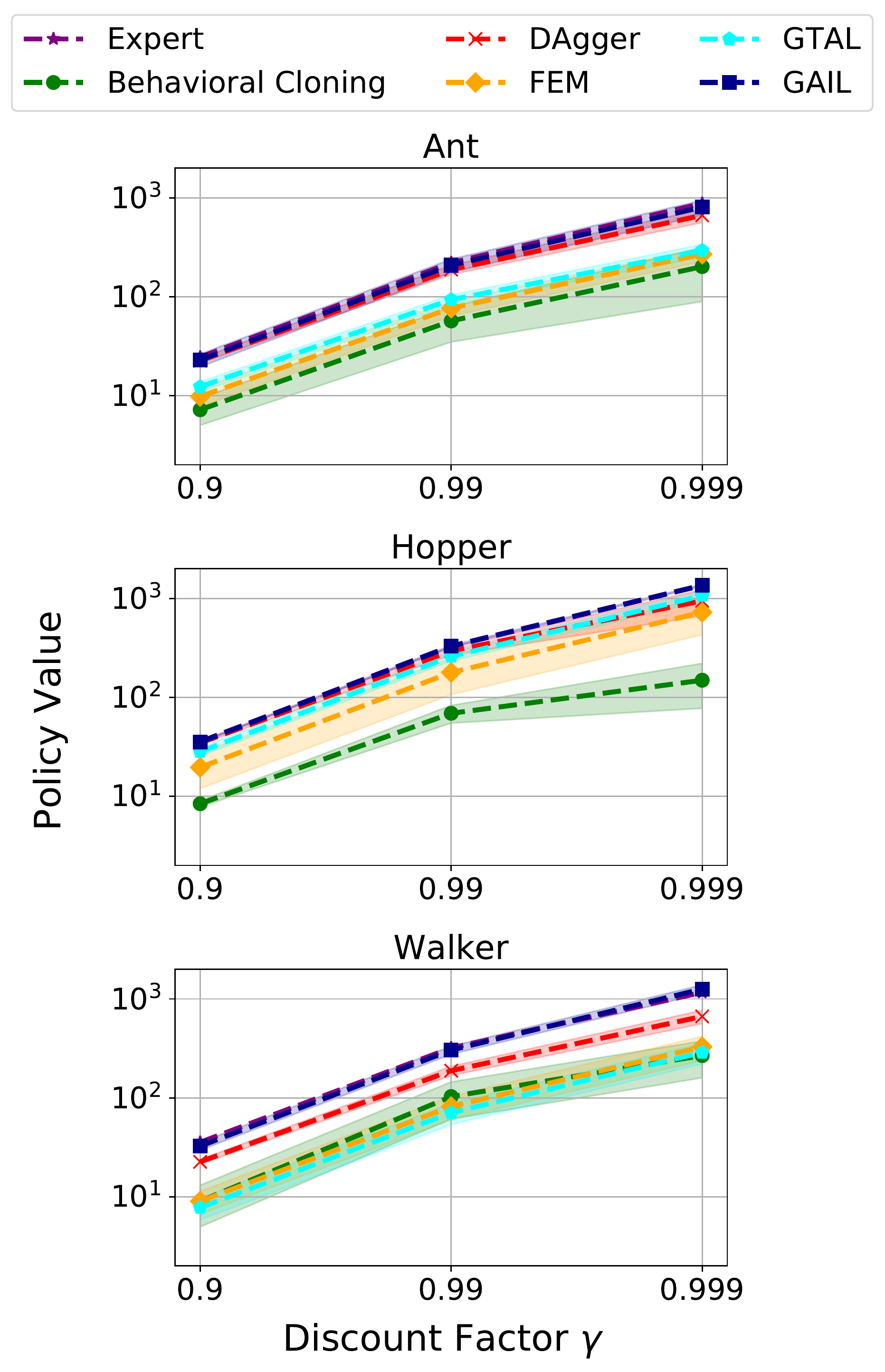}
    \caption{Policy value of learned policies ($m=25$).}
    \label{fig:policy_value}
\end{figure}
\subsection{HORIZON DEPENDENCY}
As discussed in Section \ref{sec:bc}, BC theoretically performs worse than other approaches due to the quadratic horizon dependency. The results of policy value via varying discount factor (the number of expert trajectories $m = 25$) are shown in Figure \ref{fig:policy_value}.

From Figure \ref{fig:policy_value}, we can see that as the discount factor increases, the policy value of all algorithms increases. However, the gap with the expert policy increases much quickly for BC (note that $y$-axis is $\log$ scale), especially on \textit{Hopper}. This phenomenon verifies that optimizing the discrepancy of policy distribution may not lead to a satisfying policy for sequential decision problems, as we discussed in Section \ref{sec:bc}. Though DAgger uses the same optimization objective, it presents better performance than BC thanks to querying for the expert policy when training.

\begin{figure}[t]
    \centering
    \includegraphics[width=0.5\linewidth]{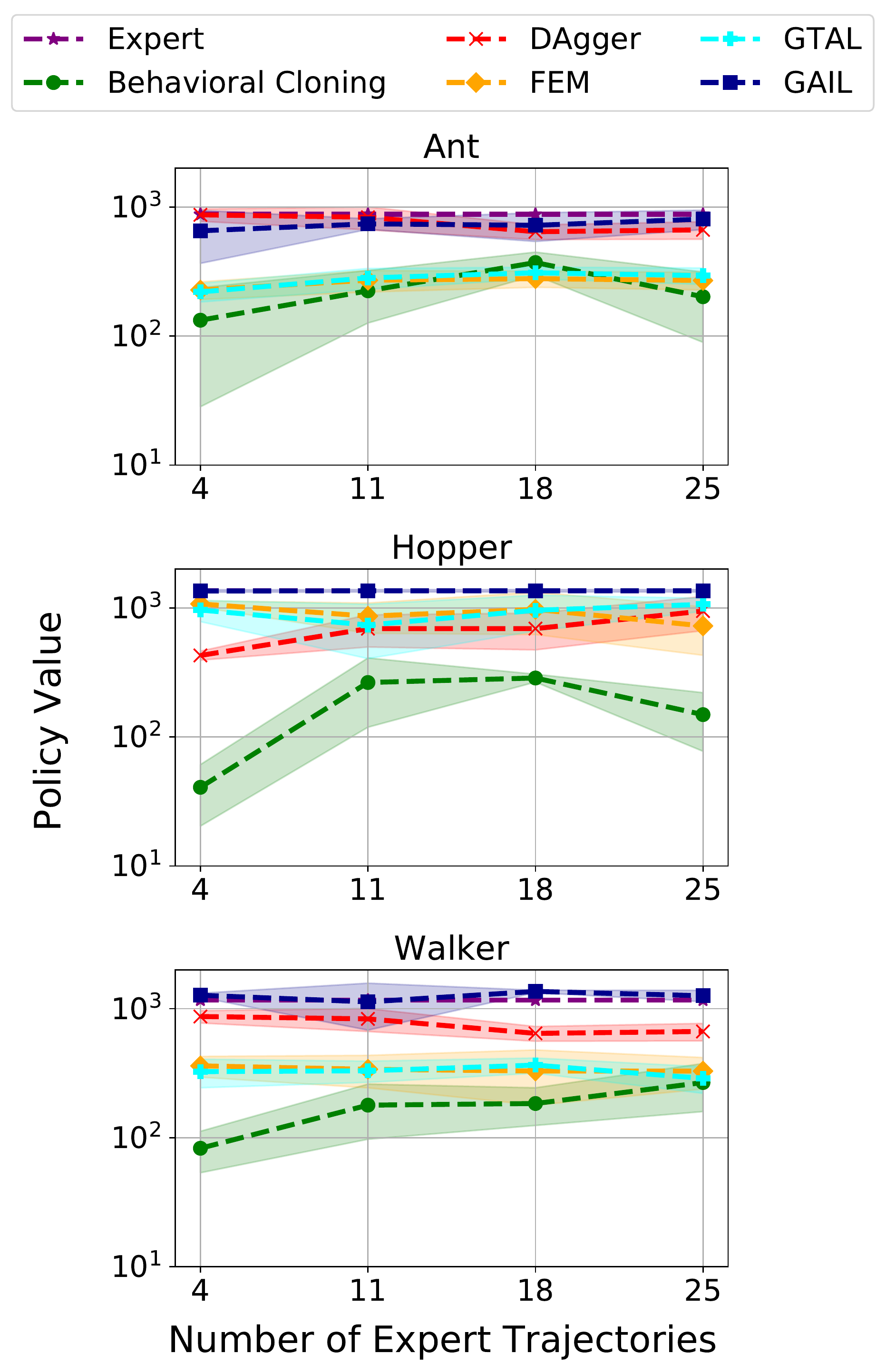}
    \caption{Policy value of learned policies ($\gamma=0.999$).}
    \label{fig:sample complexity}
\end{figure}

\subsection{Sample Complexity}
In this part, we drive into the comparison in terms of sample complexity. We report results in Figure \ref{fig:sample complexity} (discount factor $\gamma = 0.999$). Provided the same number of expert trajectories \footnote{Note that the DAgger requires 2 times expert trajectories rather than the one shown in Figure \ref{fig:sample complexity}.}, adversary-based algorithms including GAIL, FEM, MWAL, always produce better results than behavioral cloning algorithm on all environments. Interestingly, adversary-based algorithms perform well when the number of expert trajectories is small. However, these algorithms generally require more than $2M$ interactions (total interactions for GAIL, FEW, MWAL are $3M$) with the environment to reach a reasonable performance.


\section{Conclusion}
\label{sec:conclusion}

Imitation learning faces the challenge from temporally related decisions. In this paper, we propose a framework to analyze the theoretical property of imitation learning approaches based on discrepancy propagation analysis. Under the infinite-horizon setting, the framework leads to the value discrepancy of behavioral cloning in an order of $O\bigl(\frac{1}{(1-\gamma)^2}\bigr)$. We also show that the framework leads to the value discrepancy of GAIL in an order of $O \bigl( \frac{1}{1-\gamma} \bigr)$. We hope our theoretical results can provide insights for future improvements in imitation learning algorithms.

\section{Acknowledgement}

We thank Dr. Weinan Zhang for his helpful comments. This work is supported by the NSFC (61876077), Jiangsu SF (BK20170013), and Collaborative Innovation Center of Novel Software Technology and Industrialization.

\bibliography{references}
\bibliographystyle{abbrvnat}

\appendix
\section*{Appendices}
\addcontentsline{toc}{section}{Appendices}
\renewcommand{\thesubsection}{\Alph{subsection}}


\subsection{PROOF OF RESULTS IN SECTION \ref{sec:Framework}}\label{appendix_for_proof}

\noindent\textbf{Proofs for Lemma \ref{leamma_s_cond}.}
\begin{proof}
According to the definition of $\gamma$-discounted state distribution in Eq.(\ref{equation_state_distribution}), we have that
\begin{equation}
\begin{split}
 d_{\pi} &= (1-\gamma)\sum_{t=0}^{\infty} \gamma^{t} Pr(s_t=s|\pi, d_0) \\
      &=(1-\gamma)\sum_{t=0}^{\infty} \gamma^{t} P_{\pi}^{t} d_{0} \\
      &=(1-\gamma)(I - \gamma P_{\pi})^{-1} d_{0}
\end{split}
\end{equation}
where $P_{\pi} \in \mathbb{R}^{\vert \mathcal{S} \vert \times \vert \mathcal{S} \vert}$ and $P_{\pi}(s^{\prime}|s) = \sum_{a \in \mathcal{A}} P_{sa}^{s^{\prime}} \, \pi(a|s)$. Then we get that
\begin{equation}\label{d_pi}
    \begin{split}
    d_{\pi} - d_{\pi_{E}} &= (1-\gamma) [\left(I - \gamma P_{\pi}\right)^{-1} - \left(I - \gamma P_{\pi_{E}}\right)^{-1}] \; d_{0}
    \\
    &= (1-\gamma)(M_{\pi} - M_{\pi_{E}}) \; d_{0}
\end{split}
\end{equation}
Where $M_{\pi} = \left(I - \gamma P_{\pi}\right)^{-1}$ and $M_{\pi_{E}} = \left(I - \gamma P_{\pi_{E}}\right)^{-1}$.
For the term $M_{\pi} - M_{\pi_{E}}$, we get that
\begin{equation}\label{m_pi}
    \begin{split}
    M_{\pi} - M_{\pi_{E}} &= M_{\pi} \left( M_{\pi_E}^{-1} - M_{\pi}^{-1} \right) M_{\pi_{E}}
    \\
    &= \gamma (P_{\pi} - P_{\pi_{E}}) M_{\pi}
\end{split}
\end{equation}
Combining Eq. ($\ref{d_pi}$) with $(\ref{m_pi})$, we have
\begin{equation}
    \begin{split}
        d_{\pi} - d_{\pi_{E}} &= (1-\gamma) \gamma M_{\pi} \left( P_{\pi} - P_{\pi} \right) M_{\pi_E} d_{0}\\
        &= \gamma M_{\pi}\left( P_{\pi} - P_{\pi_E} \right) d_{\pi_E}
    \end{split}
\end{equation}
According to the definition of total variation and property of operator norm, we get that
\begin{equation}\label{bound_TV}
\begin{split}
        D_{\mathrm{TV}}(d_{\pi}, d_{\pi_{E}}) &= \frac{\gamma}{2} \left\| M_{\pi} (P_{\pi} - P_{\pi_{E}}) d_{\pi_{E}} \right\|_{1} \\
    &\leq \frac{\gamma}{2} \left\| M_{\pi} \right\|_{1} 
    \left\| (P_{\pi} - P_{\pi_{E}}) d_{\pi_E} \right\|_{1}
\end{split}
\end{equation}
We first show that $M_{\pi}$ is bounded:
\begin{equation}\label{bound_M}
    \left\| M_{\pi} \right\|_{1} 
    = \left\| \sum_{t=0}^{\infty} \gamma^{t} P_{\pi}^{t} \right\|_{1}
    \leq \sum_{t=0}^{\infty} \gamma^{t} \left\| P_{\pi} \right\|_{1}^{t}
    \leq \sum_{t=0}^{\infty} \gamma^{t}
    =\frac{1}{1-\gamma}
\end{equation}
Then we show that $\left\| (P_{\pi} - P_{\pi_{E}}) d_{\pi_{E}} \right\|_{1}$ is bounded:
\begin{equation}\label{bound_second}
    \begin{split}
        \left\| (P_{\pi} - P_{\pi_{E}}) d_{\pi_{E}} \right\|_{1}
        &= \sum_{s^{\prime}} \left| \sum_{s} (P_{\pi}(s^{\prime}| s) - P_{\pi_{E}}(s^{\prime}| s) ) d_{\pi_{E}}(s) \right|
        \\
        &\leq \sum_{s, s^{\prime}} \left|P_{\pi}(s^{\prime}| s) - P_{\pi_{E}}(s^{\prime}| s)  \right|d_{\pi_{E}}(s)
        \\
        &= \sum_{s, s^{\prime}} \left| \sum_{a} P_{s a}^{s^{\prime}} \left( \pi(a|s) - \pi_{E}(a|s)  \right) \right| d_{\pi_{E}}(s)
        \\
        &\leq \sum_{s, a, s^{\prime}} P_{s a}^{s^{\prime}} \left| \pi(a|s) - \pi_{E}(a|s) \right|d_{\pi_{E}}(s)
        \\
        &= \sum_{s} d_{\pi_{E}}(s) \sum_{a}\left| \pi(a|s) - \pi_{E}(a|s) \right|
        \\
        &= 2 \mathbb{E}_{s \sim d_{\pi_{E}}}[D_{\mathrm{TV}}(\pi_{E}(\cdot|s),\pi(\cdot|s) )]
    \end{split}
\end{equation}
Combining Eq.\eqref{bound_M} and (\ref{bound_second}) with (\ref{bound_TV}), we complete the proof.
\end{proof}

\subsection{PROOF OF RESULTS IN SECTION \ref{sec:bc}}\label{appendix_bc}

\noindent\textbf{Proofs for Lemma \ref{lemma_bc_pac}}.
\begin{proof}
Let $\pi_1, \dots, \pi_{\vert \Pi \vert} $ be the policy in $\Pi$. For convenience of proof, let $\hat{R}_{s}\left( \pi \right) = \frac{1}{m}\sum_{i=1}^{m} I \left[\pi\left(s^{\left(i\right)}\right) \not= a^{\left(i\right)}\right]$ \\and $R\left( \pi \right) = \mathbb{E}_{s \sim d_{\pi_{E}}(s)}[D_{\mathrm{TV}}(\pi_E(\cdot|s), \pi(\cdot|s))]$. By Hoeffding's inequality and union bound, the following inequality holds:
\begin{equation}\label{equation_gail_tv_}
    \begin{aligned}
        &P \left[ \exists \pi \in \Pi \bigl| \widehat{R}_{S}\left(\pi\right) - R\left(\pi\right) \bigr| > \epsilon \right]
        \leq {
        \sum_{\pi \in \Pi} P \left[ \bigl| \widehat{R}_{s}\left( \pi \right) - R\left( \pi \right) \bigr| > \epsilon \right]
        }\\
        &\leq{
        2 \bigl| \Pi \bigr| \exp\left( -2m\epsilon^2 \right) 
        }
    \end{aligned}
\end{equation}
Then, we can get that:
\begin{equation*}
    \forall \pi \in \Pi, P\left[ \bigl| \widehat{R}_{s}\left( \pi \right) - R\left( \pi \right) \bigr| \leq \epsilon \right] \geq 1 - 2\left| \Pi \right|exp\left( -2m\epsilon^2 \right)
\end{equation*}
Setting the right side to be equal to $1-\delta$ completes the proof.
\end{proof}

\subsection{PROOF OF RESULTS IN SECTION \ref{sec:GAIL}}\label{appendix_gail}

\noindent\textbf{Proofs for Lemma \ref{lemma_gail_neural_distance}.}
\begin{proof}
Assume that $\hat{\rho}_{\pi}$ optimizes the GAIL loss $d_{\mathcal{D}}(\hat{\rho}_{\pi_E}, \hat{\rho}_{\pi})$ up to an $\epsilon$ error.
\begin{equation}\label{equation_gail_loss}
    d_{\mathcal{D}}(\hat{\rho}_{\pi_E}, \hat{\rho}_{\pi}) \leq \underset{\pi \in \Pi}{\inf} d_{\mathcal{D}}(\hat{\rho}_{\pi_E}, \rho_{\pi}) + \epsilon
\end{equation}
With the standard derivation and Eq.(\ref{equation_gail_loss}), we prove that $d_{\mathcal{D}}(\rho_{\pi}, \rho_{\pi_E}) - \underset{\pi \in \Pi}{\inf}d_{\mathcal{D}}(\rho_{\pi_E}, \rho_{\pi})$ has an upper bound.
\begin{equation}\label{EquationGAILLemma}
    \begin{aligned}
        & d_{\mathcal{D}}(\rho_{\pi_E}, \rho_{\pi}) - \underset{\pi \in \Pi}{\inf} d_{\mathcal{D}} (\rho_{\pi_E}, \rho_{\pi})
        \\
        = & d_{\mathcal{D}}(\rho_{\pi_E}, \rho_{\pi}) - d_{\mathcal{D}}(\hat{\rho}_{\pi_E}, \hat{\rho}_{\pi}) + d_{\mathcal{D}}(\hat{\rho}_{\pi_E}, \hat{\rho}_{\pi}) - \underset{\pi \in \Pi}{\inf} d_{\mathcal{D}}(\rho_{\pi_E}, \rho_{\pi})
        \\
        \leq & d_{\mathcal{D}}(\rho_{\pi_E}, \rho_{\pi}) - d_{\mathcal{D}}(\hat{\rho}_{\pi_E}, \hat{\rho}_{\pi}) + \underset{\pi \in \Pi}{\inf}d_{\mathcal{D}}(\hat{\rho}_{\pi_E}, \rho_{\pi}) + \epsilon - \underset{\pi \in \Pi}{\inf} d_{\mathcal{D}}(\rho_{\pi_E}, \rho_{\pi})
        \\
        \leq & 2 \underset{\pi \in \Pi}{sup} \left| d_{\mathcal{D}}(\rho_{\pi_E}, \rho_{\pi}) - d_{\mathcal{D}}(\hat{\rho}_{\pi_E}, \rho_{\pi}) \right| + \epsilon
        \\
        \leq & 2 \underset{D \in \mathcal{D}}{sup} \left| \mathbbm{E}_{s, a \sim \rho_{\pi_E}}[D(s,a)] - \mathbb{E}_{s, a \sim \hat{\rho}_{\pi_E}} [D(s,a)] \right| + \epsilon
    \end{aligned}
\end{equation}

Assume that the discriminator set $\mathcal{D}$ consists of bounded function with $\Delta$, i.e. $\Delta := \underset{ D \in \mathcal{D} }{ \sup } \| D \|_{\infty} \leq \Delta$. According to McDiarmid 's inequality, with probability at least $1-\delta$, the following inequality holds.
\begin{equation*}
    \begin{aligned}
        & \underset{D \in \mathcal{D}}{\sup} \left( \mathbb{E}_{s, a \sim \rho_{\pi_E} }[D(s,a)] - \mathbb{E}_{s, a \sim \hat{\rho}_{\pi_E} }[D(s,a)] \right)
        \\
        & \leq
        \mathbb{E} \left[ \underset{D \in \mathcal{D}}{\sup} \left( {\mathbb{E}_{s, a \sim \rho_{\pi_E}}[D(s,a)] - \mathbb{E}_{s, a \sim \hat{\rho}_{\pi_E}}[D(s,a)]} \right) \right]
        + 2\Delta\sqrt{\frac{\log (1/\delta)}{2m}}
    \end{aligned}
\end{equation*}
Derived by Rademacher complexity theory, we have that
\begin{equation}\label{EquationRademacher}
    \begin{aligned}
        &\mathbb{E} \left[ \underset{D \in \mathcal{D}}{\sup} \left( {\mathbb{E}_{s, a \sim \rho_{\pi_E}}[D(s,a)] - \mathbb{E}_{s, a \sim \hat{\rho}_{\pi_E}}[D(s,a)]} \right) \right]
        \\
        \leq & \mathbb{E}_{{\bm{\sigma}}} \left[ \sup_{D \in \mathcal{D}} \sum_{i=1}^{m} \frac{1}{m} \sigma_{i} D(X_{i}) \right] 
        = 2\hat{\mathcal{R}}_{\rho_{\pi_E}}^{(m)}(\mathcal{D})
    \end{aligned}
\end{equation}

Combining Eq.(\ref{EquationGAILLemma}) with Eq.(\ref{EquationRademacher}), we complete the proof.
\end{proof}

\noindent\textbf{Proof for Lemma \ref{LemmaGailTV}}
\begin{proof}
Based on the Proposition 2.9 in \cite{D-G_tradeoff} and Pinsker 's inequality, we have that
\begin{equation}
    \frac{1}{\Delta} d_{\mathcal{D}}(\rho_{\pi}, \rho_{\pi_E}) \leq D_{\mathrm{TV}}(\rho_{\pi}, \rho_{\pi_E}) \leq \sqrt{2 \Lambda_{\mathcal{D}, \Pi} d_{\mathcal{D}}(\rho_{\pi}, \rho_{\pi_E})}
\end{equation}
Where $\Lambda_{\mathcal{D}, \Pi} = \underset{\pi \in \Pi}{\sup} \| \log (\frac{\rho_{\pi}}{\rho_{\pi_E}}) \|_{\mathcal{D}, 1} < \infty$ and\\
    {\small$\| \log (\frac{\rho_{\pi}}{\rho_{\pi_E}}) \|_{\mathcal{D}, 1} = \inf \bigl\{ \sum_{i=1}^{n} \vert w_{i} \vert : \log(\frac{\rho_{\pi}}{\rho_{\pi_E}}) = \sum_{i=1}^{n} w_{i} D_{i} + w_{0}, \forall n \in \mathbb{N}, w_{0}, w_{i} \in \mathbb{R}, D_{i} \in \mathcal{D}  \bigr\}$}.
Then we get that
\begin{equation}\label{equation_gail_tv_}
    \begin{aligned}
        &D_{\mathrm{TV}}(\rho_{\pi}, \rho_{\pi_E})
        \\
        \leq & \sqrt{2 \Lambda_{\mathcal{F}, \Pi} \left( \underset{\pi \in \Pi}{\inf} \Delta D_{\mathrm{TV}}(\rho_{\pi}, \rho_{\pi_E}) + 4 \hat{\mathcal{R}}_{\rho_{\pi_E}}^{(m)}(\mathcal{D}) + 2 \Delta \sqrt{\frac{2 \log (1/\delta)}{m}} + \epsilon \right)} 
        \\
        \leq& {
        {\sqrt{2 \Lambda_{\mathcal{D}, \Pi}}}{
        \biggl( \underset{\pi \in \Pi}{ \inf } \sqrt{\Delta D_{\mathrm{TV}}(\rho_{\pi}, \rho_{\pi_E})} +  \sqrt{\epsilon}
        }
        }
        \\
        &{ + 2 \sqrt{\hat{\mathcal{R}}_{\rho_{\pi_E}}^{(m)}(\mathcal{D})} + \sqrt{2 \Delta \sqrt{\frac{2 \log (1/\delta)}{m}}} \biggr)
        }
        \end{aligned}
\end{equation}

Consider that the number of expert demonstration is limited and confidence ratio $1-\delta$ is close to $1$, we notice that$\sqrt{2 \Delta \sqrt{\frac{2 \log (1/\delta)}{m}}} \leq 2 \Delta \sqrt{\frac{2 \log (1/\delta)}{m}} $.
Combining it with Eq.\eqref{equation_gail_tv_}, we conclude the proof.
\end{proof}

\noindent\textbf{Proof for Theorem \ref{theorem_gail}.}
\begin{proof}
Recall the definition of value function, we notice that
$\left| V^{\pi_E} - V^{\pi}  \right| \leq \frac{2 R_{\max}}{1-\gamma} D_{\mathrm{TV}}(\rho_{\pi}, \rho_{\pi_E})$.
Combining with Lemma \ref{LemmaGailTV}, we have the following result.
{
\small\begin{equation}
    \begin{aligned}
        \left| V^{\pi} - V^{\pi_E} \right|
        \leq & 
        {
        \frac{2 R_{\max}}{1-\gamma} D_{\mathrm{TV}}(\rho_{\pi}, \rho_{\pi_E})
        }
        \\
        \leq & {\frac{2 R_{\max}  \sqrt{2 \Lambda_{\mathcal{D}, \Pi}}}{1-\gamma}
        \biggl( \underset{\pi \in \Pi}{\inf} \sqrt{\Delta D_{\mathrm{TV}}(\rho_{\pi}, \rho_{\pi_E})} + \sqrt{\epsilon} 
        }
        \\
        &+ 
        {2 \sqrt{\hat{\mathcal{R}}_{\rho_{\pi_E}}^{(m)}(\mathcal{D})}  + \sqrt{2 \Delta \sqrt{\frac{2 log(1/\delta)}{m}}}  \biggr)}
    \end{aligned}
    \end{equation}
}
Derived by the connection between total variation and JS divergence shown in Eq.(\ref{equation_connection_JS_TV}), we complete the proof.
\end{proof}

\end{document}